\documentclass[final]{article}
\usepackage{arxiv}
\usepackage[utf8]{inputenc}
\usepackage[inline]{enumitem}
\usepackage{amsmath}
\usepackage{amssymb}
\usepackage{amsthm}
\usepackage{algorithm}
\usepackage{algpseudocode}
\usepackage[obeyFinal,colorinlistoftodos]{todonotes}
\usepackage{hyperref}
\usepackage{setspace}
\usepackage{svg}
\usepackage{balance}
\usepackage{tikz}

\graphicspath{{fig/}}

\title{Byzantine Resilience at Swarm Scale: A Decentralized Blocklist Protocol from Inter-robot Accusations}
\renewcommand{\shorttitle}{Decentralized Blocklist Protocol}
\renewcommand{\headeright}{preprint}
\renewcommand{\undertitle}{preprint}
\author{
    Kacper Wardega\\
    Boston University\\
    Boston\\
    USA\\
\texttt{ktw@bu.edu}\\
\And
Max von Hippel\\
    Northeastern University\\
    Boston\\
    USA\\
\texttt{vonhippel.m@northeastern.edu}\\
\And
Roberto Tron\\
    Boston University\\
    Boston\\
    USA\\
\texttt{tron@bu.edu}\\
\And
Cristina Nita-Rotaru\\
    Northeastern University\\
    Boston\\
    USA\\
\texttt{c.nitarotaru@northeastern.edu}\\
\And
Wenchao Li\\
    Boston University\\
    Boston\\
    USA\\
\texttt{wenchao@bu.edu}\\
}

\newcommand{\closedbox}{\mathcal{U}}
\newcommand{\fastbox}[1]{\closedbox_{d(t-s_#1)}(\tilde{x}_#1)}

\newcommand{\lilwork}{p}

\newcommand{\lilprop}{x}

\newcommand{\robots}{V}
\newcommand{\comm}{E}
\newcommand{\coop}{\mathcal{C}}
\newcommand{\uncoop}{\bar{\mathcal{C}}}

\newcommand{\accmsg}{\texttt{Acc}}
\newcommand{\accg}{\mathcal{A}}
\newcommand{\removed}{R}

\newcommand{\argos}{\textsc{ARGoS}}
\newcommand{\paratitle}[1]{{\noindent\bf #1}.}

\newtheorem{theorem}{Theorem}

\newtheorem{remark}{Remark}
\newtheorem{definition}{Definition}

\begin{document}

\maketitle
\begin{abstract}
    
The Weighted-Mean Subsequence Reduced (W-MSR) algorithm, the state-of-the-art method for Byzantine-resilient design of
decentralized multi-robot systems, is based on discarding outliers received over Linear
Consensus Protocol (LCP). Although W-MSR provides well-understood
theoretical guarantees relating robust network connectivity to the convergence of the underlying consensus, the
method comes with several limitations preventing its use at scale:
\begin{enumerate*}\item the number of Byzantine robots, $F$, to tolerate
should be known a priori, \item the requirement that each robot maintains $2F+1$ neighbors is impractical for large $F$, \item information propagation is hindered by the requirement that $F+1$ robots independently make local measurements of the consensus property in order for the swarm's decision to change, and \item W-MSR is specific to LCP and does not
generalize to applications not implemented over LCP\end{enumerate*}. In this
work, we propose a Decentralized Blocklist Protocol (DBP) based on inter-robot
accusations. Accusations are made on the basis of locally-made observations of misbehavior, and once shared by cooperative robots across the network are used as input to a graph matching algorithm that computes a blocklist. DBP generalizes to applications not implemented via LCP, is
adaptive to the number of Byzantine robots, and allows for fast information
propagation through the multi-robot system while simultaneously reducing the required network
connectivity relative to W-MSR. On LCP-type applications, DBP reduces the worst-case connectivity requirement of W-MSR from $(2F+1)$-connected to $(F+1)$-connected and the minimum number of cooperative observers required to propagate new information from $F+1$ to just $1$ observer. We demonstrate empirically that our approach to
Byzantine resilience scales to hundreds of robots on cooperative target
tracking, time synchronization, and localization case studies.

\end{abstract}
\section{Introduction}
\label{sec:introduction}

Multi-robot systems (MRS) presently employed in industry use structured deployment environments and highly centralized designs~\cite{verma_multi-robot_2021}. Central coordination benefits all key MRS components -- task allocation, execution, fault detection and recovery, while structured environments allow for strict physical security measures. In contrast, emergent MRS applications in unstructured environments (such as patrol, search and rescue, coverage, shape formation, and collective transport) are typically not amenable to centralized approaches due to communication constraints~\cite{gielis_critical_2022}. Decentralized methods to mitigate the negative impact of faulty and/or malicious robots in unstructured environments have therefore attracted much research attention, especially since a wide range of attacks have been shown to disrupt MRS function and safety, e.g. sensor perturbation and denial-of-service (DoS)~\cite{djouadi_finite_2015, zhou_distributed_2020, liu_distributed_2021}, actuator jamming~\cite{guo_roboads_2018}, networking DoS~\cite{yaacoub_robotics_2022}, or Sybil/fraudulent identity attacks~\cite{gil_guaranteeing_2017,mallmann-trenn_crowd_2021}. Given the multitude of possible attacks, it is important to understand the resilience of the MRS from \emph{Byzantine attackers} -- that is if an unknown subset of the robots is allowed to have arbitrarily different behaviors relative to the cooperative robots in terms of physical actions and communication.

Byzantine-unaware MRS implementations are often highly vulnerable, and break completely, when even one robot has been comprised. In our case studies for example, Byzantine robots may cause robots within a swarm to follow a false target, or have arbitrarily large errors in time synchronization or localization. The main approach proposed for Byzantine-resilient MRS is the \emph{Weighted-Mean Subsequence Reduced} (W-MSR) algorithm~\cite{leblanc_resilient_2013,saulnier_resilient_2017,mitra_resilient_2019}. W-MSR is easy to implement and has well-understood theoretical guarantees. However, W-MSR can only be used for MRS applications that are implemented via Linear Consensus Protocol (LCP), performance does not scale with the number of robots in the system, and the number of Byzantine robots to tolerate, $F$, is a parameter that must be known a priori. Suppose that LCP is the means by which the robots reach a collective decision about a physical property of the environment. The choice of $F$ in W-MSR dictates how many outliers robots should discard in each update of linear consensus; each robot needs at minimum $2F+1$ neighbors in order to update their local consensus variable and at minimum $F+1$ cooperative robots must independently make direct measurement of the underlying physical quantity. If $F$ is chosen smaller than the number of Byzantine robots, then the mitigation provided by W-MSR is forfeit. For large $F$, the network connectivity requirement and the logistics of maintaining $F+1$ cooperative observers renders W-MSR impractical. 

In this work we propose \emph{Decentralized Blocklist Protocol} (DBP), an approach to Byzantine resilience inspired by P2P networks, based on inter-robot accusations. Cooperative robots make use of local observations to detect misbehaving peers and make accusations accordingly. Accusations propagate through the cooperative robots, which each robot then independently processes with a matching algorithm to compute a blocklist. We derive necessary and sufficient conditions on the set of accusations that must be made and connectivity of the MRS that ensures that all Byzantine robots are eventually blocked by the cooperative robots, and their influence mitigated. Specifically, we show that for a closed MRS satisfying an analogous $(F+1)$-connectivity requirement for time-varying networks, blocking all of the Byzantine robots is equivalent to Hall's marriage condition on the accusations made within the system.  In addition to W-MSR requiring the number of Byzantine robots to tolerate be known a priori, we claim that W-MSR does not scale with the number of robots in practice. We show empirically on target tracking and time synchronization applications that this is the case, and that our proposed approach adaptively scales to hundreds of robots/attackers, in contrast to just one or two attackers in a swarm of no more than 20 robots as in related works.  W-MSR cannot be used to provide Byzantine resilience for MRS not implemented over LCP, such as cooperative localization. We implement Byzantine-resilient cooperative localization using our approach as a proof of concept; to our knowledge ours is the first successful technique for decentralized and Byzantine-resilient cooperative localization.

\begin{figure}
    \centering
    \includegraphics[width=\linewidth]{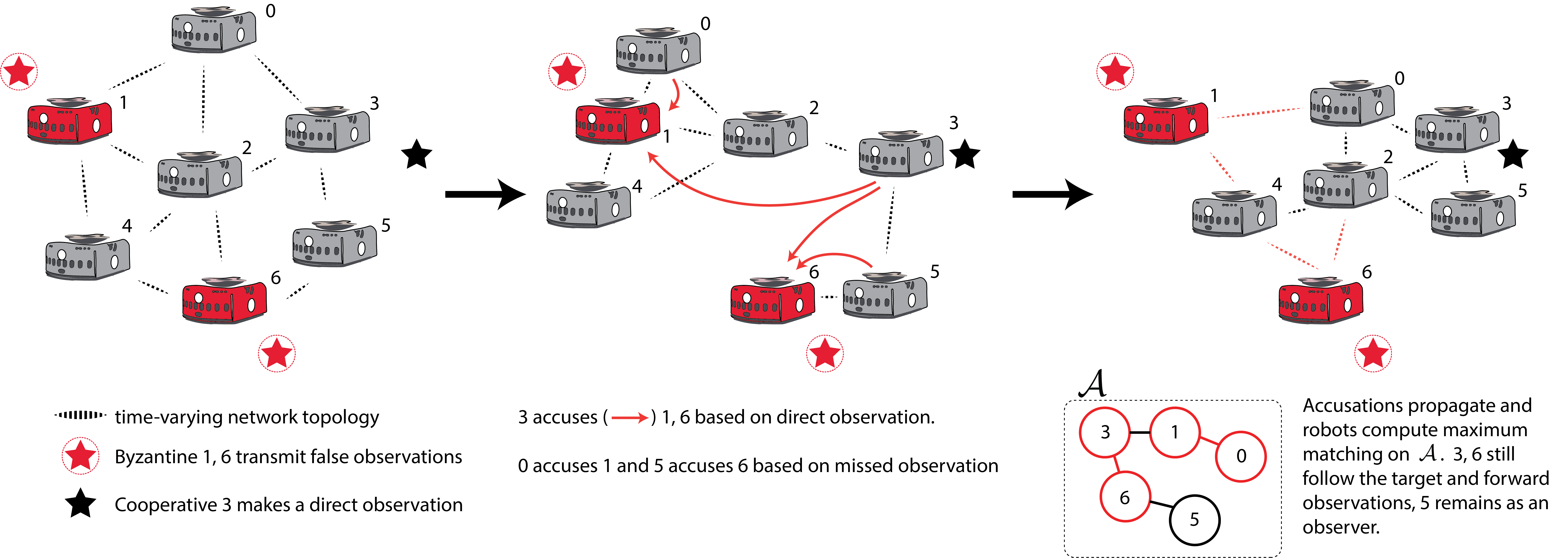} %
    \caption{DBP is used to provide Byzantine resilience for a simple seven robot (two Byzantine) scenario of the target tracking case study explored in Section~\ref{subsec:target-tracking}. Byzantine robots 1 and 6 transmit false target observations (red stars) while cooperative robot 3 makes a direct observation of the target (black star). Based on the closest observer, the cooperative robots move towards the supposed target location. Robots 0 and 5 make it to the false locations reported by 1 and 6 respectively and accuse the observers on the basis of missed observations that should have been made. Meanwhile, robot 3 accuses both 1 and 6 since their observations contradict 3's direct observation. The accusations flood through the cooperative robots, until eventually each cooperative robot has accusation graph denoted $\accg$. Edmond's algorithm is used by each robot to independently compute the maximum matching $\{(3,6),(0,1)\}$ (red edges), thus observations from robots 0, 1, 3, and 6 are blocked. Observations from robots 2, 4, and 5 are still trusted by the other cooperative robots. Note that 0 and 3 still continue to cooperate by moving towards the target and forwarding observations from non-blocked observers.}
    \label{fig:vigilante-overview}
\end{figure}

\section{Background \& Related Work}
\label{sec:related}

\paratitle{W-MSR} Perhaps the most well-understood approach to Byzantine-resilient decentralized MRS is the W-MSR algorithm. W-MSR can be applied to MRS applications that are implemented over Linear Consensus Protocol -- a distributed consensus algorithm for real-valued variables whereby in each timestep robots update their local variable to a convex combination of their neighbor's broadcast values, i.e. $$x_i(t)=\sum_{j\in\mathcal{N}_G(i)}\alpha_jx_j(t-1)\text{ where }\sum\alpha_j=1$$ and $\mathcal{N}_G(i)$ are the neighbors of $i$ in the connectivity graph $G$. The authors of~\cite{leblanc_resilient_2013} first introduced W-MSR for Byzantine resilience which discards the $F$ highest and $F$ lowest values received at each timestep of LCP, and show that convergence despite up to $F$ Byzantine robots is equivalent to a graph robustness property. 
 Specifically, if the connectivity graph of the robots is at least $(2F+1)$-vertex-connected, then the consensus will converge to a value within the convex hull of the cooperative robots' initial values.
W-MSR has been applied to a variety of applications, such as flocking~\cite{saulnier_resilient_2017} and state estimation~\cite{mitra_resilient_2019}.
Extensions for the W-MSR algorithm to time-varying networks where the union of the connectivity graphs within a bounded window is robust are proposed in~\cite{saldana_resilient_2017} and to event-driven control in~\cite{amirian_distributed_2021}.
Methods to form robust graph topologies, as required by the W-MSR algorithm, are proposed in~\cite{guerrero-bonilla_formations_2017}.

\paratitle{Blockchain} Distributed ledger technologies, e.g. blockchains, have also attracted much research attention for its potential to provide resilience guarantees. For similar settings as considered in this paper, \cite{strobel2018managing} proposes to use an Ethereum blockchain for a MRS collective decision-making case study. The authors of  \cite{pacheco2020blockchain} investigate the approach of consensus over blockchain.
We refer the reader to a recent survey \cite{aditya2021survey} of work on blockchain for robotics applications, including for MRS and swarm.

\paratitle{Inter-robot observations} The use of inter-robot observations to detect misbehavior and establish trust is a common theme in multi-agent systems generally. As opposed to our work, where accusations are used to compute a blocklist, the authors of \cite{ashkenazi_forgive_2019} propose that cooperative robots should take physical action to isolate misbehaving robots on observing incorrect behavior. In a cooperative patrolling case study for example, cooperative robots surround and impede the movement of robots that are observed not following the correct trajectory. More commonly, inter-robot observations are used as input to a reputation mechanism, whereby robots maintain real-valued reputation scores of their peers. For example, \cite{fagiolini_decentralized_2007} presents a connected vehicle case study where robots use partial information to determine if their local neighbors are non-cooperative, and a later work \cite{arrichiello_observer-based_2015} proposes an adaptive threshold-based actuator fault detection strategy for MRS. The use of reputation scores as input to a  cooperative coverage problem is explored in \cite{inaba_adaptive_2016}, and 
\cite{cheng2021general} introduces a general trust framework for multi-robot systems with case studies on intelligent intersection management. Reputation mechanisms can also be merged with consensus, i.e. \cite{fagiolini_consensus-based_2008} proposes that robots perform consensus on the reputation values of the robots. Ultimately, reputation mechanisms inherit the drawbacks of W-MSR, in that $F+1$ cooperative robots would be required to assign a low reputation to each misbehaving robot before their influence is removed from the system, and furthermore the number of consensus variables (the reputation scores) scales linearly with the size of the swarm. 

\paratitle{MRS security} Beyond the mentioned directly related works, our work is motivated by the broader push for secure MRS. We refer the reader to \cite{bijani_review_2014} for a comprehensive survey of open security problems in MRS, with problem-specific recommendations for mitigation, and to  \cite{zhou_multi-robot_2021} for a treatment of recent approaches to MRS in uncertain or adversarial operating environments.
Decentralized MRS has a wide attack surface; \cite{guo_roboads_2018} tests CPS-inspired anomaly detection for a variety of compromise scenarios such as wheel jamming, LiDAR denial-of-service, wheel encoder logic fault, etc., and \cite{djouadi_finite_2015} analyzes bounded sensor attacks. A novelty of our work is the proof-of-concept Byzantine-resilience for cooperative localization. Byzantine localization is a challenging problem, \cite{weber_gordian_2020} derives conditions under which Byzantine localization is solvable in the centralized setting; similarly as with W-MSR, Byzantine localization requires the graph structure of inter-robot measurements to satisfy a robustness property.

\paratitle{Sybil attacks} Certain threat models have received special treatment in the literature. Efficient methods for scheduling MRS under denial of service attacks are proposed in \cite{zhou_distributed_2020}, which \cite{liu_distributed_2021} extends to the decentralized setting. In our work, we assume that the MRS is protected from Sybil attacks since a central trust authority issues identities for all of the robots. We believe that Sybil-proofness of the system via central authority is a reasonable assumption for a \emph{closed} MRS, where the set of robots does not change with time.  Decentralized identity management and Sybil-proofness for MRS is however an active area of research orthogonal to our own. Using inter-robot radio signals to detect Sybil identities is first proposed in \cite{gil_guaranteeing_2017}; the same technique is later applied to a cooperative flocking scenario under Sybil threat model in \cite{mallmann-trenn_crowd_2021}. Other physics-inspired, application-specific approaches have been proposed in the literature, for example defending against Sybil attacks on a crowdsourced traffic light \cite{shoukry_sybil_2018}. Prior work has also proposed to incorporate Sybil attack prevention as a component of the W-MSR algorithm \cite{renganathan_spoof_2017}.

\section{Threat Model}
\label{sec:threat}

We consider a swarm robotics system with robots connected by time-varying network topology  $G[t]=(\robots,\comm[t])$.   Each robot has an identity that has been issued by a trusted central authority at deploy-time, which it uses to both send
signed messages to its neighbors and to verify the authenticity of received messages.  Assume that some unknown subset of the robots have been compromised by a Byzantine adversary. We refer to the cooperative robots as $\coop\in\robots$,  the Byzantine ones as $\uncoop=\robots\setminus\coop$, and we assume that the sets $\robots$ and $\coop$ are fixed, i.e. the MRS is \emph{closed}. We assume a strong adversary,
where the Byzantine robots can coordinate centrally with each other online, have detailed knowledge of the system implementation such as robot capabilities and application details, can send arbitrary messages to the cooperative robots, and have arbitrary physical behaviors. The goal of the Byzantine robots is to disrupt the MRS application; the specific goal and attack strategy will depend on the application. Each of our case studies in Section~\ref{sec:case} will specify the attacker's goal and strategy. Since the robots use their trust authority-issued identities to communicate, we assume that Sybil attacks are not possible for the adversary, since the adversary is unable to forge fraudulent identities that will be accepted by $\coop$. However, robots whose
identities (secret keys) have been compromised can be used by the adversary to send misleading messages. Therefore, any robots in the swarm whose keys have been compromised are considered to be part of $\uncoop$. 

\section{Decentralized Blocklist Protocol}
\label{sec:DBP}

Decentralized Blocklist Protocol (DBP) is a swarm blocklist algorithm that is adaptive to the presence of Byzantine adversaries. DBP can be used as an alternative to W-MSR, but with lower requirement on network connectivity and without needing to know~$F$ ahead of time. DBP is adaptive, and as such the requirement on robust network topology scales with the true number of Byzantine robots. The connectivity requirements of W-MSR scale with the parameter~$F$, even if the actual number of Byzantine robots is lower. An example of how DBP works on a target tracking scenario is shown in Fig.~\ref{fig:vigilante-overview}. Based on locally-made observations, cooperative robots accuse misbehaving peers. The accusations propagate through the network via flooding and are used as input to a matching algorithm that outputs a blocklist.

DBP relies on flooding as a networking primitive, where cooperative robots always re-broadcast (forward) received messages.  Messages in DBP are accusations signed by the robot initiating the flood. Accusations  $\accmsg_i(j)$ are an application-agnostic message and the payload is simply the identity of a robot $j$ that the origin $i$ wishes to accuse. The precise rules used to decide if and when an accusation should be issued are application-specific. Accusations serve to remove the influence of Byzantine nodes on the swarm application. Each robot $i$ locally maintains a set $\removed_i[t]$ of accusations that it has received. A subset $\removed^*_i[t]\subseteq\removed_i[t]$ will be locally computed by $i$ using any deterministic maximum matching algorithm (such as Edmond's~\cite{edmonds1965paths})   to form the blocklist. For the remainder
of this section, we will assume that the robots  have a \emph{sound} accusation mechanism:

\begin{definition}[Sound accusation]
    If a cooperative robot $i$ issues an accusation $\accmsg_i(j)$, then $\accmsg_i(j)$ is \emph{sound} if and only if $j\in\uncoop$. 
\end{definition}

\begin{remark}
    \label{rem:lor}
    In the presence of Byzantine robots, receiving a message $\accmsg_i(j)$ implies that $i\in\uncoop\lor j\in\uncoop$. The reason is that if $i\in\coop$, then $j\in\uncoop$ by soundness of accusations. In the other case, $i\in\uncoop$.
\end{remark}

 \paratitle{Matching} Importantly, the set of received accusations $\removed_i[t]$ has a structure imparted by the accusation soundness. Given an undirected graph $G=(V=X\cup Y,E)$ with $X,Y$ disjoint, we say that $G$ is \emph{$X$-semi-bipartite} if $X$ is an independent vertex set in $G$. A subset $\mathcal{M}\subseteq E$ is a \emph{matching} on $G$ if $\mathcal{M}$ is an independent edge set in $G$. Given a matching $\mathcal{M}$, we denote by $V_\mathcal{M}$ the matched vertices in $\mathcal{M}$. If no additional edges can be added to a matching $\mathcal{M}$, then $\mathcal{M}$ is \emph{maximal}. If there does not exist a matching $\mathcal{M}^*$ s.t. $|V_{\mathcal{M}^*}|>|V_\mathcal{M}|$, then $\mathcal{M}$ is a maximum cardinality, or \emph{maximum}, matching.  Given a subset $S\subseteq V$, a matching $\mathcal{M}$ is \emph{$S$-perfect} if $S\subseteq V_\mathcal{M}$. The following condition allows us to connect the notion of maximum matching and perfect matching:
 
 \begin{definition}[Hall's Marriage condition]
	Given $G=(X\cup Y,E)$ s.t. $G$ is $X$-semi-bipartite, a $Y$-perfect matching exists if $\forall S\subseteq Y$, $|S|\leq|\mathcal{N}_G(S)\cap X|$. Additionally, any maximum matching will be $Y$-perfect.
\end{definition}

\paratitle{Accusation graph} Now let $\accg_k[t]$ be the \emph{accusation graph} with edge $(i,j)$ iff $\accmsg_i(j)\in\removed_k[t]$. As we note in Remark~\ref{rem:lor}, each accusation can be viewed as a disjunction -- $\accmsg_i(j)$ can be understood as ``$i$ is Byzantine or $j$ is Byzantine (or both are).'' Therefore, $\accg_k[t]$ is $\coop$-semi-bipartite, and any matching $M$ on $\accg_k[t]$ will satisfy $|V_M|\leq 2|\uncoop|$. The inequality will be tight if and only if the Hall marriage condition holds for $\uncoop$ on $\accg_k[t]$ -- in which case the maximum matching $M$ is $\uncoop$-perfect with $|V_M|=2|\uncoop|$. Robot $k$ chooses $\removed^*_k[t]$ to be the matched vertices of the maximum matching on $\accg_k[t]$ -- the robots corresponding to the matched vertices are the ones that $k$ will block. An example accusation graph and associated maximum matching is shown as ``$\accg$'' in Fig.~\ref{fig:vigilante-overview}.

\paratitle{Network flooding} This matching result is only useful if the requisite accusations actually propagate through the robots in $\coop$. Given a time-varying directed graph $G[t]=(V,E[t])$, consider the execution of a   network flood where a node $v\in V$ initiates a flood at time $\tau$ by transmitting a message to its neighbors $\mathcal{N}_{G[\tau]}(v)$.   The flood continues when $v$'s neighbors transmit to their neighbors so that at time $\tau+2$, $\mathcal{N}_{G[\tau+1]}(\mathcal{N}_{G[\tau]}(v))$ will receive the message. Continuing the pattern, the $s$-frontier of the flood, for positive integer $s$, is given by $$\mathcal{N}^s_{G[\tau]}(v):=\mathcal{N}_{G[\tau+s-1]}(\mathcal{N}_{G[\tau+s-2]}(\cdots\mathcal{N}_{G[\tau]}(v)))$$ The $s$-closure of the flood is then the union $$\mathcal{N}^{s^*}_{G[\tau]}(v):=\mathcal{N}^0_{G[\tau]}\cup\cdots\cup\mathcal{N}^s_{G[\tau]}$$
If for arbitrary initial node $v$ and starting time $\tau$, there exists a positive integer $s$ such that $\mathcal{N}^{s^*}_{G[\tau]}(v)=V$, then we say that $G[t]$ is \emph{floodable}. So far we have assumed that nodes may re-transmit the message multiple times. If we limit the number of re-transmissions to $n$, and there still exists an $s$ s.t. the analogously defined $(n,s)$-closure equals $V$, then we say that $G[t]$ is \emph{$n$-floodable}. If $|V|\geq k$ and after the removal of an arbitrary set of $k$ nodes from $V$, $G[t]$ is still $n$-floodable, then we say that $G[t]$ is \emph{$(k,n)$-floodable}. Ultimately, we can now state that cooperative nodes will eventually hear all accusations and have the same accusation graph despite up to $F$ Byzantine robots:
\begin{theorem}[Eventual Blocklist Consensus]
	Let $G[t]$ be the time-varying, $(F,n)$-floodable network topology of the robot swarm $\robots$. If $|\uncoop|\leq F$, there $\exists\tau\in\mathbb{Z}^+,\accg\forall i\in\coop,s\geq\tau$ s.t. $\accg=\accg_i[s]$. 
\end{theorem}
\begin{proof} By definition of $(F,n)$-floodable, we have that all accusations made by $V$ will eventually reach all of $\coop$, since the cooperative robots can ensure eventual delivery of an accusation to all of $\coop$ even if up to $F$ Byzantines do not forward accusations.  Given that the MRS is closed, the number of possible accusations is finite (bounded by $2|\coop|+|\uncoop|^2$) and therefore there exists a time $\tau$ after which no new accusations are made. As each cooperative robot uses the same deterministic algorithm to compute maximum matchings on the accusation graph, each cooperative robot will eventually compute the same maximum matching and arrive at the same list of robots to block.\end{proof}

If the assumption that $G[t]$ is $(F,n)$-floodable does not hold, then some cooperative robot(s) may not receive 
some of the accusations. If the $\removed_k[t]$ are not eventually equivalent across all $k$, it is possible that not all uncooperative robots are blocked (even though globally, enough accusations have been made to satisfy the Hall marriage condition). However, all of $\uncoop$ will be blocked by $k$ provided that $k$'s local accusation graph $\accg_k[t]$ satisfies the Hall marriage condition, but the matched cooperative robots on the blocklist may not be the same as those on other blocklists. 

\section{Case Studies}
\label{sec:case}

We run our experiments on turtlebots simulated in \argos~\cite{Pinciroli:SI2012},
a multi-physics robot simulator that can efficiently simulate large-scale swarms of robots. The robots are equipped with a radio to transmit to neighboring robots within 4m and have an omnidirectional camera used for nearby target detection and collision avoidance with an observation distance of $\approx0.9$m. The robot controller runs at 30Hz. Source code to reproduce our experiments can be found at \url{https://github.com/gitsper/decentralized-blocklist-protocol}

\subsection{Target Tracking}
\label{subsec:target-tracking}

\paratitle{Application overview} In swarm target tracking, the goal of the robots is to locate and cooperatively follow a mobile target that has a maximum speed of $d$. In our experimental setup, the target is a robot that has a yellow light -- robots within a distance $r$ can see the light and make a direct observation of the target. 
To enable the entire swarm to track the target, even for those robots that do not directly observe the target, robots broadcast target observation messages containing:
\begin{enumerate}
    \item the observer's unique ID
    \item the time of the observation
    \item the observed location of the target
\end{enumerate}
In each timestep, robots sort received observation messages by observation time, and choose the most recent one to  transmit to its neighbors. Robots keep track of how many times a given observation message has been transmitted, and stop sharing it after fixed, finite number of times. The purpose of transmitting the same observation message multiple times is to account for the time-varying connectivity with neighboring robots. In addition to the application messages, DBP is used to mitigate the influence of Byzantine robots. Robots delete and do not forward observations messages from blocked observer IDs. Old observation messages are periodically deleted from the local cache.

\begin{figure}[ht]
    \centering
    \usetikzlibrary{decorations.markings,arrows,calc,shapes.geometric,positioning}
\begin{tikzpicture}[decoration={markings,
        mark=between positions 0 and 0.9 step 0.25
            with {
                        \draw (0,0) node[
                                                    rectangle,
                                                    fill=blue!40,
                                                    draw=black,
                                                    semitransparent,
                                                    name=rect-\pgfkeysvalueof{/pgf/decoration/mark info/sequence number},
                                                    minimum height=100-16*\pgfkeysvalueof{/pgf/decoration/mark info/sequence number},
                                                    minimum width=100-16*\pgfkeysvalueof{/pgf/decoration/mark info/sequence number}] {};
                         \node[opaque, star, star points=5, star point ratio=1.8,inner sep=1.2pt, fill=black, draw] at (0,0) {};
                    },
            mark=at position 0.999 with {\arrow[opaque,black]{>};}
        }]
    \draw [transparent,postaction={decorate}] (0,0) .. controls (0.2,0.7) and (0.8,1.5) .. (1.2,0.8);
    \draw (0,0) .. controls (0.2,0.7) and (0.8,1.5) .. (1.2,0.8);
    \draw[->] ([xshift=-40pt,yshift=20pt]rect-1.west) -- ([xshift=-3pt]rect-1.west) node[yshift=5pt,xshift=-5pt,left,above,near start] {$\mathcal{U}_{d(t-s)}(\tilde{x})$};
    \draw (2.4,0.5) node[rectangle,fill=red!40,draw=black,semitransparent,minimum height=24,minimum width=24,name=rect-bad] {};
    \node[font=\footnotesize,align=left,right of=rect-bad,anchor=west,xshift=-4pt] (bad-text) {Greedily ignore \\observations with\\empty intersection};
    \draw[->] ([xshift=2pt]rect-bad.east) -- (bad-text.west);
    \node[opaque,star,star points=5, star point ratio=1.8,inner sep=1.2pt, fill=red!80] at (rect-bad) {};
    \draw[draw=red!80,densely dotted] (rect-bad) circle (3.4pt);
\end{tikzpicture}
    \caption{Observation-based target tracking setup for use with DBP. Robots that do not observe the target directly sort received observations by age and compute a bounding box for each observation containing the target based on the elapsed time. Reducing over the bounding boxes with the set intersection operator yields the robot's current belief about the target location. Conflicting observations, those that result in an empty intersection, are dropped, ending the iteration.}
    \label{fig:target-tracking-motion}
\end{figure}
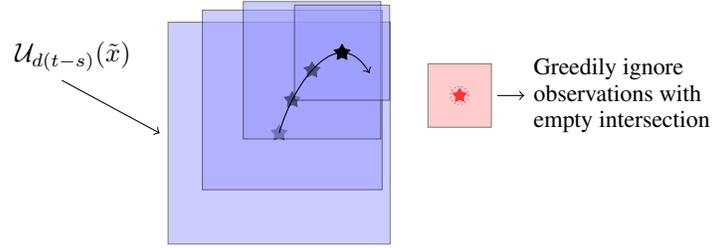

\paratitle{Controller} For robots that directly observe the target, they compute a heading vector pointing to the target from their current location and move towards the target. Robots that do not directly observe the target rely on received observation messages to compute their heading vector. We denote by $\closedbox_d(c)$ the closed square centered at $c$ with side length $2d$. Given an observation message with  time $s$ and observed target location $\tilde{x}$, the implied belief is that the set $\closedbox_{d(t-s)}(\tilde{x})$ contains the target at the current time $t>s$. First, the received observation messages are sorted by time $(s_1,\tilde{x}_1),(s_2,\tilde{x}_2),\ldots$ with $s_1\geq s_2\geq\cdots$. To compute the heading vector, robots iteratively take the intersection $$\fastbox{1}\cap\fastbox{2}\cap\cdots$$ If the intersection ever  becomes empty while iterating,  the offending observation is dropped and the iteration ends. Robots take the center of the intersection to be their believed target location and use it to compute their heading vector. The control procedure is illustrated in Fig.~\ref{fig:target-tracking-motion}. Bounding boxes are used instead of circles to simplify the computation of set intersections.%

\paratitle{Accusation rules} On receiving a new observation message, robots issue DBP accusations according to four target tracking-specific accusation rules. Given the received observation by robot $j$ of $\tilde{\lilprop}$ made at time $s$, let $\Delta t=t-s$ the elapsed time, $\Delta\lilwork_i=\|\lilwork_i[t]-\tilde{\lilprop}\|$ the distance from $i$'s location $\lilwork_i[t]$ to the observed target, and $c$ a constant denoting an upper bound on the speed with which messages can travel through the network (in our experimental setup, 4m/timestep). The first accusation rule is triggered when $r+c\Delta t<\Delta\lilwork_i$, as the observation would need to have traveled faster-than-possible through the network. The second accusation rule is triggered when $\Delta\lilwork_i<r-d\Delta t$ and $i$ did not make a direct observation of the target -- $i$ missed an observation that it should have made if the received observation was legitimate. The third accusation is rule is triggered when $\Delta\lilwork_i>r+d\Delta t$ but a direct observation \emph{was} made by $i$; in this case the target couldn't possibly have moved fast enough from the received observation location to the place where $i$ observed it presently. Finally, the last accusation rule detects oscillations from a single observer. If $i$ has received an observation from $j$ in the past, it will consider the most recent previous observation from $j$ of $\tilde{\lilprop}_\text{old}$ at time $s_\text{old}$, and will make an accusation of $j$ if $\|\tilde{\lilprop}-\tilde{\lilprop}_\text{old}\|>d(s-s_\text{old})$. In this case, $j$'s observations are inherently inconsistent with the maximum rate of change in $\lilprop$.

\paratitle{Experiment setup} We compare DBP-based Byzantine-resilient target tracking with the state-of-the-art W-MSR-based approach. Aside from not needing to know the number of Byzantine robots to tolerate a priori and lower network connectivity requirement, \emph{our approach requires just one non-blocked cooperative robot to observe the moving target}, whereas W-MSR requires $F+1$ cooperative observers to shift the consensus among the cooperative robots as the target moves. We simulate $|\coop|=200$ and $|\uncoop|=100$ robots to compare tracking performance. Byzantine robots may transmit observation messages and accusations with arbitrary contents. In our scenarios, the behavior of the Byzantine robots is to distribute evenly through the environment and to continuously broadcast false observations -- each Byzantine robot picks the location $\sim0.4$m away from itself directed away from the origin as the broadcast observation. This Byzantine strategy attempts to lower the network connectivity by causing the cooperative robots to spread out and away from the origin, while simultaneously not violating the speed of network accusation rule.

\begin{figure}[ht]
    \centering
    \includegraphics[width=0.4\linewidth]{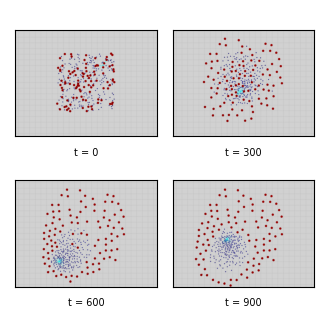}
    \caption{View of DBP-based target tracking in \argos. Byzantine robots are highlighted with red circles, direct observations of the target are shown in cyan.}
    \label{fig:target-tracking-DBP-vis}
\end{figure}
\begin{figure}[ht]
    \centering\def\svgwidth{0.6\linewidth}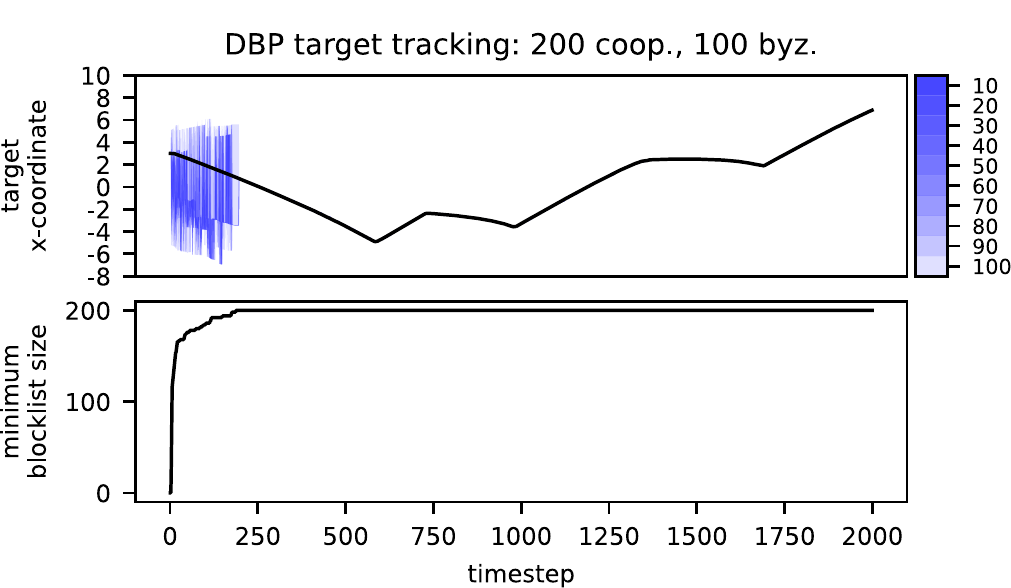
    \caption{DBP-based target tracking performance. At top, the black curve shows the true x-coordinate of the moving target and the shaded blue regions show the range of beliefs as percentiles around the median. At bottom we plot $\min_i\removed_i^*[t]$, i.e. the minimum blocklist size. At timestep $\sim200$, all of the Byzantine robots have been blocked on each cooperative robot, and the cooperative robots track the target with close to no error as the influence of the Byzantines has been removed.}
    \label{fig:target-tracking-DBP}
\end{figure}
\begin{figure}[ht]
    \centering\def\svgwidth{0.6\linewidth}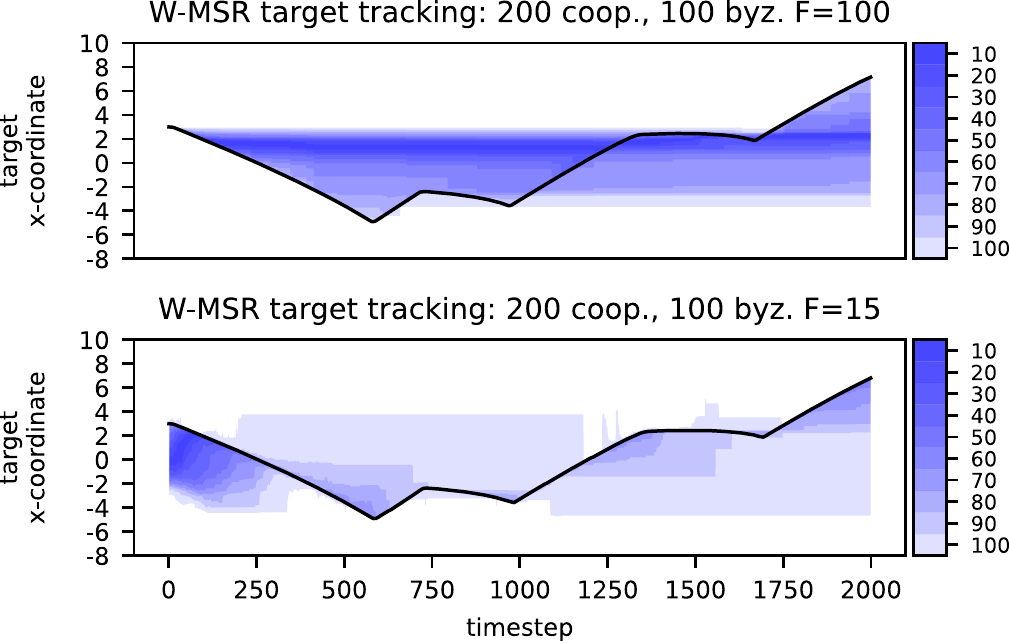
    \caption{W-MSR-based target tracking performance. When the resilience parameter $F=100$ (top) in order to guarantee safety, the information about the moving target cannot propagate through the cooperative robots due to the high connectivity and simultaneous observer requirements. As a comparison, $F=15$ (bottom) has no safety guarantee but allows a subset of the robots to track the target successfully. However, the influence of the Byzantines is never removed.}
    \label{fig:target-tracking-WMSR}
\end{figure}

\paratitle{Experiment result} In Fig.~\ref{fig:target-tracking-DBP} we plot the belief that each cooperative robot has about the $x$-coordinate of the target, summarized using a quantile heatmap. The range of beliefs decreases until approx. $t=400$, at which point all of the Byzantine robots have been blocked and the execution enters the regime with all Byzantine influence removed. Views of the DBP target tracking experiment in the \argos~simulator are shown in Fig.~\ref{fig:target-tracking-DBP-vis}. The baseline W-MSR algorithm requires the resilience parameter $F$ to be picked a priori. If $F$ is chosen too small, the theoretical guarantees of W-MSR are forfeited so the cooperative robots' consensus may be disrupted by the Byzantine robots. Specifically, part of the swarm where the density of Byzantine robots is low may be able to track the target successfully, however cooperative robots with more than $F$ Byzantine neighbors will be affected by the attack. Each cooperative robot affected by the attack will in turn strengthen the attack as their local value nears the attacker's value -- this scenario is shown with $F=15<|\uncoop|$ at the bottom of Fig.~\ref{fig:target-tracking-WMSR}. However, W-MSR does not scale to large $F$, since the robots cannot achieve such a high level of network connectivity and also high number of cooperative observers. The large-$F$ regime is shown at top in Fig.~\ref{fig:target-tracking-WMSR}, with $F=100=|\uncoop|$.

\subsection{Time Synchronization}

\paratitle{Application overview} For this task, the robots' objective is to cooperatively synchronize their local clocks to a universal reference clock while moving through the environment. A subset of the robots are designated as \emph{anchors} -- these robots periodically make high-precision observations of the reference clock time. As in the target tracking application, the anchors broadcast observation messages containing: \begin{enumerate} \item the observer's unique ID \item the observed time \end{enumerate} In each timestep, the non-anchor robots sort received observation messages by the observed time in decreasing order and choose the largest value to re-transmit to neighbors, and DBP is used to delete and selectively not forward observation messages from blocked observers.

\paratitle{Controller} On those timesteps when new observation messages are received, non-anchor robots simply update their local clock by setting it to the maximum observed time in their list of observation messages. If a new observation message is not received during a timestep, a non-anchor robot $i$'s local clock is updated by adding a number sampled from the distribution $1+\mu_i+U[-0.05,0.05]$, where $U[a,b]$ is the uniform distribution on $[a,b]$ and $\mu_i$ is sampled at the beginning of the simulation from $U[-0.01,0.01]$. This update behavior is intended to simulate a random-walk clock drift when no new observation messages are received.

\paratitle{Accusation rules} Whenever an anchor robot receives a new observation message, it issues an accusation of the origin if the observed time is larger than the anchor's local time. The intuition behind this accusation rule is that the difference between the received observed time and the anchor's local time can only be negative -- if the observer is cooperative then the difference should correspond to the number of hops that the observation made on a shortest path to the receiving anchor. If the difference were to be positive, this would imply that the observer's local clock is ahead, violating the assumption that cooperative anchors make high-precision observations of the reference time.

\paratitle{Experiment setup} We compare DBP-based Byzantine-resilient time synchronization with the state-of-the-art W-MSR-based approach. We simulate $|\coop|=100$ ($50$ of which act as anchors, with observation period of 100 timesteps) and $|\uncoop|=45$ robots to compare the synchronization performance. Byzantine robots may send arbitrary observation messages, including impersonating anchors. The behavior of the Byzantines in our experiments is to move through the environment just as the cooperative robots do, while broadcast false reference clock observations with the same period as the cooperative anchors. The false observations are the true reference clock value, plus an attack offset of $+1000$ timesteps. This choice of Byzantine adversary attempts to disrupt the time synchronization of the cooperative nodes by forcing the non-anchors to adopt local clock values that are too large -- too-low values would be ignored by cooperative robots since each non-anchor always sets their local clock to the maximum observed clock value.

\begin{figure}[h]
    \centering\def\svgwidth{0.6\linewidth}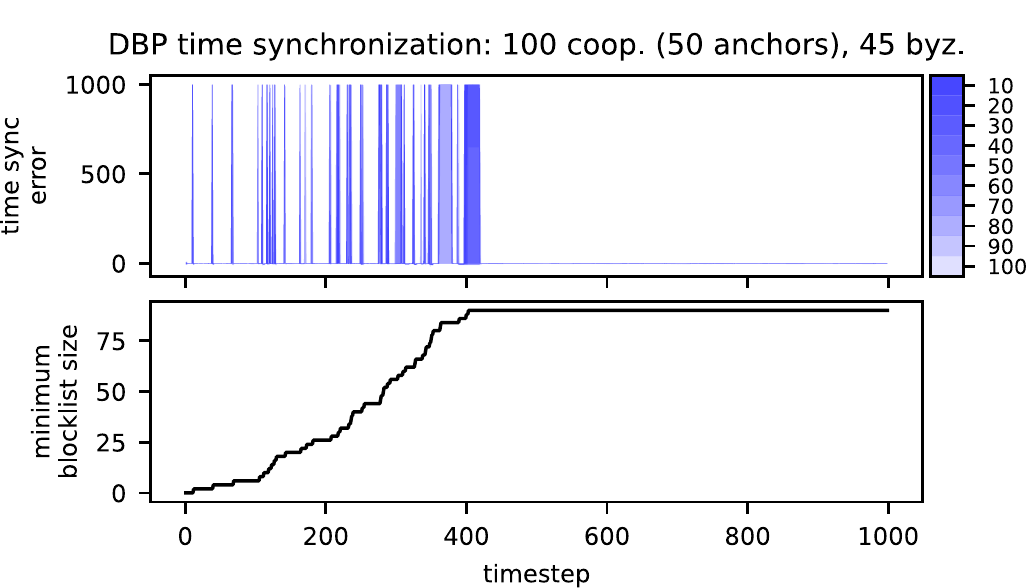
    \caption{DBP-based time synchronization performance. Shown at top is the error between the cooperative robots' local clocks relative to the global reference clock. The error spikes to the attacker offset whenever a Byzantine robot initiates an attack. After timestep $\sim400$, all of the Byzantines have been blocked and the tracking error remains nominal.}
    \label{fig:time-sync-DBP}
\end{figure}
\begin{figure}[h]
	\centering\def\svgwidth{0.6\linewidth}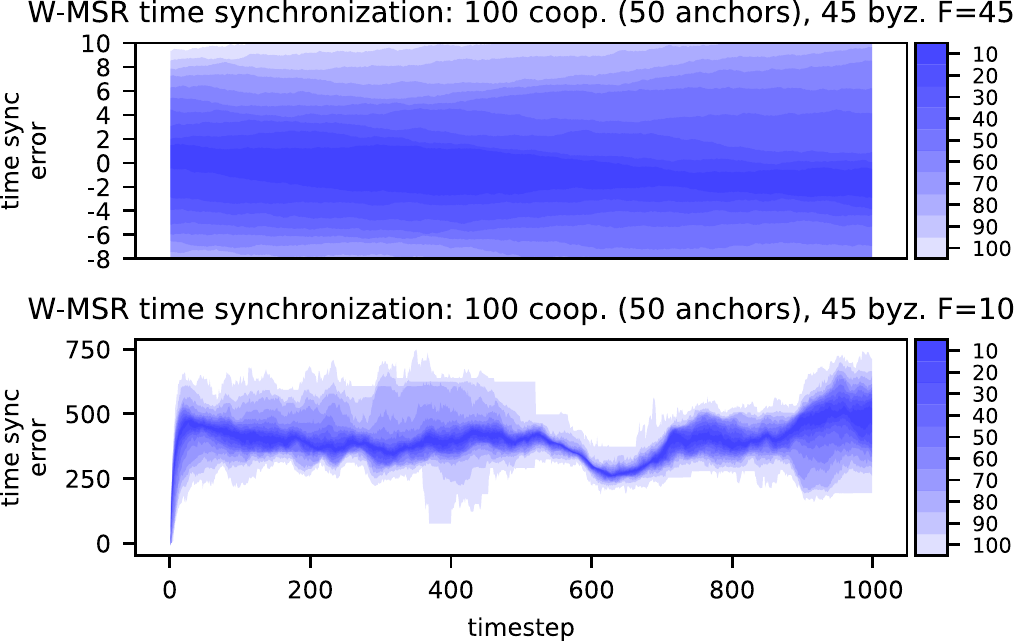
	\caption{W-MSR-based time synchronization performance. As in the target tracking case study, $F=|\uncoop|=45$ guarantees safety, but the associated connectivity requirement prohibits convergence. Conversely, a lower $F=10$ permits convergence of the consensus at the cost of allowing the Byzantines to adversely perturb the cooperative robots' local clock values.} 
	\label{fig:time-sync-WMSR}
\end{figure}

\paratitle{Experiment result} In Fig.~\ref{fig:time-sync-DBP} we plot the time synchronization error (difference between local time and the reference time) of the cooperative non-anchor robots over the course of the simulation. We observe that the Byzantine robots are able to push the synchronization error to the attack offset of $+1000$ timesteps by transmitting the false reference time observations, until timestep $\sim400$, at which point each cooperative robot has blocked all of the Byzantine robots and the attacker influence is successfully removed. As with the target tracking case study, the W-MSR baseline requires a choice to be made a priori for the resilience parameter, $F$. If $F$ is chosen too small, e.g. $F=10$ shown at bottom in Fig.~\ref{fig:time-sync-WMSR}, then the Byzantine robots influence the consensus and the cooperative robots have local clock values between the attack offset and the reference time. If $F$ is chosen large enough to be resilient to $|\uncoop|=45$ attackers, then the connectivity and simultaneous observation requirement is too large for the non-anchor robots to update their local clocks from neighbor's observations. The large-$F$ regime is shown at top in Fig.~\ref{fig:time-sync-WMSR} with $F=45=|\uncoop|$.

\subsection{Cooperative Localization}

\paratitle{Application overview} In the cooperative localization task, robots move in an unknown and/or dynamic environment and use local inter-robot distance measurements to estimate their position within a global coordinate system. To facilitate this task, a subset of the robots operate as anchors, and periodically make high-precision observations of their position (e.g. as static, pre-positioned anchors or mobile robots with GPS). As opposed to the target tracking and time synchronization applications, non-anchor robots also broadcast a localization message containing their localization belief. The localization message contains:
\begin{itemize}
	\item the sender's unique ID
	\item the sender's local time
	\item the sender's believed localization, expressed as bounding box
	\item an anchor flag, set if and only if the sender is an anchor\begin{itemize} \item if the anchor flag is not set, the most recently received anchor localization message\end{itemize}
\end{itemize} 
Non-anchor robots initially have no belief about their localization. Once a belief is formed (initially, just the anchors), non-anchors begin to periodically broadcast localization messages to their neighbors. The anchor flag will be set only if the sender is an anchor. Localization messages from non-anchors will be ignored unless the message includes an attached localization message with the anchor flag set.

\paratitle{Controller} On those timesteps when localization messages are received, non-anchor robots sort received localization messages by the time of the underlying anchor message (most recent first), and then use a stable sort to sort by anchor flag (anchor messages first). After sorting, the robot iterates over the received localization messages and takes the intersection of each localization belief, dilated by the transmission range, $c$, plus the maximum distance a robot can travel per timestep, $d$. If the intersection ever becomes empty while iterating, the last localization message is dropped and the iteration ends. The resulting intersection is the bounding box that represents the robot's new localization belief. In the next timestep, the robot will transmit its localization belief, bundling the most recent anchor message encountered during the iteration (this may be a direct transmission from an anchor, or an anchor message that arrived as an attachment to a non-anchor's message). The algorithm's operation is illustrated in Fig.~\ref{fig:coop-loc-belief}. DBP is used to delete and ignore messages from blocked senders.

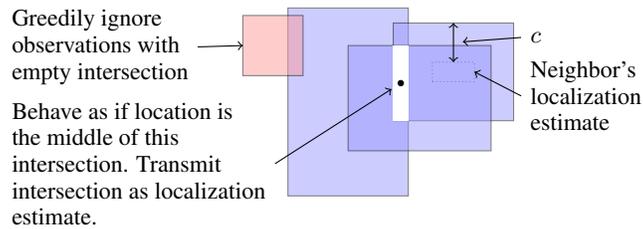
\begin{figure}[ht]
    \centering
    \usetikzlibrary{decorations.markings,arrows,calc,shapes.geometric,positioning}
\begin{tikzpicture}
    \draw[fill=blue!40,draw=black,semitransparent] (-1.4,-1) rectangle (0.2,1.5);
    \draw[fill=blue!40,draw=black,semitransparent] (0,0) rectangle (1.6,1.3);
    \draw[densely dotted, semitransparent] (-0+.52,-0+.52) rectangle (1.6-.52,1.3-.52);
    \draw[fill=blue!40,draw=black,semitransparent] (-0.6,-0.4) rectangle (1.3,1);
    \draw[fill=white,opaque,draw opacity=0] (0,0) rectangle (0.2,1) node[midway] (loc) {};
    \draw[fill=black] (loc) circle (1pt);
    \draw[<->] (1.6/2,1.3) -- ({(1.6-.52)/2+0.52/2},1.3-.52) node[midway] (offset) {};
    \node[name=text-offset,font=\footnotesize,align=left,anchor=north west] at (1.7,1.3) {$c$};
    \draw[->] (text-offset.west) -- (offset);
    \node[name=text-obs,font=\footnotesize,align=left,anchor=north west] at (text-offset.south west) {Neighbor's \\ localization \\ estimate};
    \draw[->] (text-obs.west) -- (1.6-.52,{(1.3-.52)/2+.52/2});
    \draw[fill=red!40,draw=black,semitransparent] (-2.0,0.6) rectangle (-1.2, 1.4);
    \node[name=text-bad,font=\footnotesize,align=left,anchor=north east] at (-2.6,1.6) {Greedily ignore \\ observations with \\ empty intersection};
    \draw[->] (text-bad.east) -- (-2,{(1.4-.6)/2+.6});
    \node[name=text-loc, font=\footnotesize,align=left,anchor=north west] at (text-bad.south west) {Behave as if location is \\ the middle of this \\intersection.  Transmit \\intersection as localization\\ estimate.};
    \draw[->] (text-loc.east) -- (loc);
\end{tikzpicture}
	\caption{Observation-based cooperative localization setup for use with DBP. Non-anchor robots estimate their localization based on localization estimates received from their neighbors. Estimates are dilated by the transmission distance, and then reduced with the set intersection operator to compute the localization belief. Localization estimates are ordered by the age of the underlying anchor message, with estimates sent directly from anchors given priority.}
	\label{fig:coop-loc-belief}
\end{figure}

\paratitle{Accusation rules} Received localization messages are subjected to two accusation rules. The first rule is applied by anchor robots when receiving localization messages from other anchors, either directly or as attachments to non-anchor localization messages. Given that the other anchor $j$ claims to be at $\tilde{\lilwork}_j$ at time $s$, let $\Delta t=t-s$ the elapsed time and $\Delta\lilprop_i=\|\tilde{\lilwork}_j-\lilwork_i\|$. The receiving anchor $i$ will accuse $j$ if $c\Delta t<\Delta\lilprop_i$, or in other words, if the anchor $j$'s localization message has traveled faster-than-possible through the network. The second accusation rule can be issued by all robots, including non-anchors. The second rule asserts that the first rule hold between any received non-anchor localization message and its attached anchor message. These simple accusations could be extended if the robot capabilities were better. For example, if the robots could measure a lower bound on the distance from senders, anchors would be able to issue analogous accusations in situations where localization messages from other anchors should have been received sooner.

 \paratitle{Experiment setup} The W-MSR algorithm cannot be chosen as a baseline for this case study, as cooperative localization is not solved via linear consensus problem outside of small-scale settings where each robot can directly observe every other robot in the swarm. We instead demonstrate our approach as a proof-of-concept for Byzantine-resilient cooperative localization. We simulate $|\coop|=120$ (80 of which act as fixed-position anchors) and $|\uncoop|=50$. The Byzantine robots, which attempt to disrupt the localization of the cooperative non-anchors, transmit false anchor localization messages by taking their true position and adding a random attack offset to the x- and y-coordinates sampled uniformly from [-20,20]m. The impact of the false anchor messages on non-anchor robots is to disrupt the iteration over localization messages -- since the false anchor localization will likely have an empty intersection with localization messages from nearby cooperative anchors, leading to degraded cooperative localization performance.

\begin{figure}[ht]
    \centering\def\svgwidth{0.6\linewidth}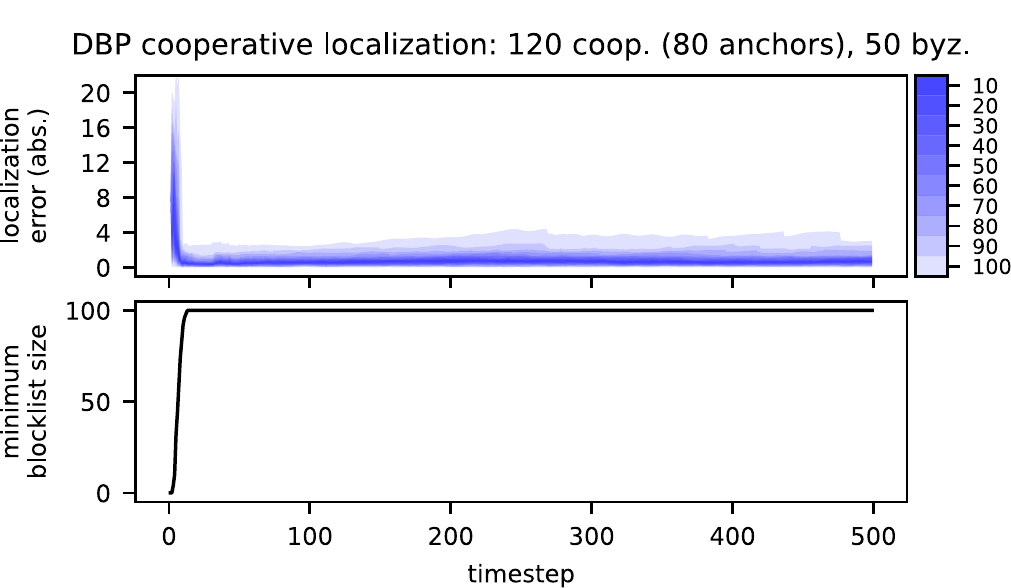
    \caption{DBP-based cooperative localization performance. At top, we plot the absolute error that the cooperative robots have in the estimate of the x-coordinate of their position. At bottom we plot the minimum size of the cooperative robots' blocklists -- once all of the Byzantines are blocked the estimation error returns to nominal values as the influence of the Byzantines has been mitigated.}
    \label{fig:coop-loc-DBP}
\end{figure}
\begin{figure}[ht]
    \centering\def\svgwidth{0.6\linewidth}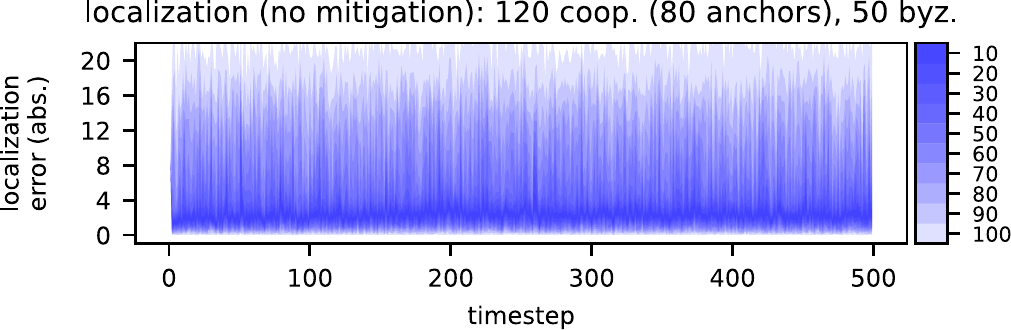
    \caption{To support our claim that DBP is a suitable approach for this task, we show the impact that Byzantine robots can have on cooperative localization -- Byzantines can cause the cooperative robots to have arbitrarily large localization errors.}
    \label{fig:coop-loc-no-mit}
\end{figure}

\paratitle{Experiment result} In Fig.~\ref{fig:coop-loc-DBP} we plot the absolute error that the cooperative non-anchor robots have in their x-coordinate, i.e. the absolute difference between what they believe their x-coordinate to be and the ground truth. We observe that while initially the cooperative non-anchors may have errors near the attack offset of $\sim20$m, the Byzantine robots are rapidly accused and blocked by the cooperative robots. After the Byzantine robots have been blocked, the anchor localization sharing algorithm provides low-error cooperative localization for the non-anchor robots. As a point of comparison, we also simulate the same scenario with DBP disabled, with the absolute x-coordinate localization error shown in Fig.~\ref{fig:coop-loc-no-mit}. As expected, the Byzantine robots significantly disrupt the localization, causing the cooperative non-anchor robots to have consistently high errors up to the attack offset.

\section{Conclusion}
\label{sec:conclusion}

This work has proposed the use of a decentralized blocklist protocol based on inter-robot accusations as a means to provide Byzantine resilience for multi-robot systems. We have shown that as an alternative to the W-MSR algorithm, our approach permits temporary Byzantine influence while accusations are made, but in exchange adapts to Byzantine robots as they are detected, allows for fast information propagation, and can be applied for applications beyond consensus. Based on empirical evidence from swarm target tracking, time synchronization, and localization case studies, our approach is more practical than W-MSR in terms of scalability to large swarms as it does not require each cooperative robot to have $2F+1$  neighbors, nor does it require $F+1$ cooperative observers for information to propagate. In fact, our approach only requires that messages are delivered by network floods in spite of $F$ Byzantine robots, and observations from a single cooperative robot can propagate quickly through the entire swarm. Furthermore, we have shown that our approach can for the first time provide Byzantine resilience for the large-scale decentralized cooperative localization problem. In our future work, we hope to extend our approach to systems where accusations are not always sound and to explore swarm algorithms that optimize the speed with which Byzantine robots are discovered and accused.

\balance
\bibliographystyle{unsrtnat}
\bibliography{bib/main,bib/vigilante}

\end{document}